\documentclass{article}

\usepackage[preprint]{caisc_2026}

\usepackage[utf8]{inputenc} 
\usepackage[T1]{fontenc}    
\usepackage{hyperref}       
\usepackage{url}            
\usepackage{booktabs}       
\usepackage{amsfonts}       
\usepackage{nicefrac}       
\usepackage{microtype}      
\usepackage{xcolor}         
\usepackage{amsmath}
\usepackage{amssymb}
\usepackage{graphicx}
\usepackage{amsthm}
\newtheorem{proposition}{Proposition}
\title{Emergent Ordinal Geometry in Transformers Trained on Local Comparisons}

\author{%
  Nishit Singh\thanks{\href{https://singhnishit.github.io/singhnishit/}{Personal webpage}} \\
  Birla Institute of Technology and Science, Pilani\\
  India\\
  \texttt{f20221317@pilani.bits-pilani.ac.in} \\
}

\begin{document}

\maketitle

\begin{abstract}
  Transitive inference is the challenge of inferring that A < C from knowing only adjacent relations (A < B, B < C). It is solved by humans and animals not through logical chaining but via an analogue mental number line, whose signature is the symbolic distance effect: distant comparisons are easier than nearby ones. We ask whether Transformers acquire the same primitive, training small models exclusively on adjacent comparisons from a hidden total order and evaluating generalization to unseen distant pairs. We find that out-of-distribution generalization emerges alongside a striking geometric reorganization: entity embeddings collapse onto a one-dimensional manifold whose principal axis recovers the hidden rank order with near-perfect fidelity, and this structure is sensitive to optimization in ways that produce grokking-like transient dynamics. Critically, even when accuracy is at ceiling, decision confidence and geometric separation both scale monotonically with rank distance, directly mirroring the symbolic distance effect observed across decades of behavioural experiments on humans, primates, and rodents. We further show the same rank-aligned geometry in a pretrained large language model, where it tracks the topology of each ordinal relation: linear for sizes and digits, cyclic for months. These results ground a 50-year-old behavioural regularity in the geometry of learned representations, offering a mechanistic account of transitive inference that bridges cognitive science and modern neural networks.
\end{abstract}

\section{Introduction}

Knowing that Alice is taller than Bob and Bob taller than Carol, one infers that Alice is taller than Carol without direct comparison. This capacity, \textit{transitive inference} (TI), is phylogenetically widespread and documented in children (\cite{Bryant1971-tp}), primates (\cite{Mcgonigle1977-qz}), rodents (\cite{Davis1992-ui}), and birds (\cite{Steirn1995}) and, crucially, is not solved by formal logic. Behavioral work instead supports an analog mental number line: items are arranged along a continuous internal ordering from which comparisons are read off spatially (\cite{RILEY19761}). The signature of this account is the symbolic distance effect: distant comparisons are faster and more accurate than adjacent ones (\cite{Bryant1971-tp}), the opposite of what step-by-step chaining predicts, and exactly what one expects if order is encoded as position on a line.

Transformers are under no explicit pressure to represent order this way, yet abstract relational structure is known to emerge in their representations as a byproduct of training (\cite{minegishi2026emergentanalogicalreasoningtransformers, mikolov-etal-2013-linguistic}), sometimes abruptly and late, in grokking-like transitions (\cite{power}). We ask a question at the intersection of cognitive science and interpretability: when a Transformer learns to order items from only local comparisons, does it discover the same primitive, \textit{a mental number line}, and exhibit the same behavioral signatures as biological cognition?

Mirroring the original \cite{Bryant1971-tp} paradigm, in which subjects learn adjacent comparisons and are tested on inferred distant pairs, we train small Transformers only on adjacent relations from a hidden total order and evaluate on held-out distant pairs. We find that out-of-distribution generalization coincides with a geometric reorganization: entity embeddings collapse onto a one-dimensional manifold whose principal axis recovers the hidden rank order. The resemblance is behavioral as well as structural, even at accuracy ceiling, decision confidence and geometric separation scale with rank distance, a computational analogue of the symbolic distance effect. Our contributions are:
\textbf{(i)} a minimum-norm account predicting the geometry, including end-anchoring;
\textbf{(ii)}  A synthetic task, modeled on the classic TI paradigm, isolating ordinal generalization from local comparisons. \textbf{(iii)} Evidence that Transformers spontaneously form a 1D ordinal manifold and reproduce the symbolic distance effect. \textbf{(iv)} evidence that the same geometry appears in a pretrained 1.5B LLM, tracking the topology of each relation.

\section{Background}
\paragraph{Transitive inference and the distance effect:}
Since \cite{Bryant1971-tp}, TI has been studied as ordinal rather than logical reasoning; the symbolic distance and end-anchor effects are its canonical behavioral signatures across species (\cite{Steirn1995}).

\paragraph{Emergent structure in Transformers:}
Generalizing solutions can appear abruptly after memorization, a phenomenon termed grokking (\cite{power}), driven by weight decay and trackable through internal progress measures as discussed in texts like \cite{nanda2023}; \cite{minegishi2026emergentanalogicalreasoningtransformers} formalize relational reasoning as emergent representational geometry.

\paragraph{Linear representations:} Semantic and relational structure is often encoded linearly in embedding space (\cite{park2025iclrincontextlearningrepresentations, elhage2022toymodelssuperposition}), motivating our use of principal-component analysis as a probe.

\section{Theoretical Framework}

\label{sec:theory}

The rank-ordered one-dimensional manifold is not an accident of this task.
It is the configuration that $L_2$ regularization \emph{must} prefer (\cite{neyshabur2015searchrealinductivebias}).

We make this precise under a standard idealization: the comparison is read
from a linear functional of the entity embeddings. We then show that the
predicted geometry matches our empirical findings, including a characteristic
departure from uniformity that turns out to be a second cognitive signature.

\subsection*{Setup}

Let the relation be decided by a fixed readout direction $w\in\mathbb{R}^d$ with
$\lVert w\rVert=1$, through the scalar \emph{score}
\[
  s_i \;=\; \langle w,\, h_{e_i}\rangle .
\]
The model predicts $e_i\prec e_j$ whenever $s_j - s_i > 0$. This score is exactly
the quantity our PC1 projection $\pi_i$ estimates empirically.

Training supervises only adjacent pairs. We encode this as a margin constraint
with target margin $\gamma>0$:
\begin{equation}
  s_{i+1}-s_i \;\ge\; \gamma,
  \qquad 0\le i<N-1 .
  \label{eq:margin}
\end{equation}

Weight decay then selects, among all feasible embeddings, the one of minimum norm:
\begin{equation}
  \min_{\{h_{e_i}\}}\ \sum_{i=0}^{N-1}\lVert h_{e_i}\rVert^2
  \qquad\text{subject to \eqref{eq:margin}.}
  \label{eq:minnorm}
\end{equation}

\subsection*{The regularized solution is the rank line}

\begin{proposition}
\label{prop:line}
Every minimizer of \eqref{eq:minnorm} satisfies:
\begin{enumerate}
  \item[\textnormal{(i)}] \textbf{Collinearity:}
        $h_{e_i}=s_i\,w$ for all $i$; every off-axis component vanishes.
  \item[\textnormal{(ii)}] \textbf{Monotonicity:}
        $s_0<s_1<\dots<s_{N-1}$, so PC1 recovers the hidden order
        ($r_{\mathrm{rank}}=1$).
  \item[\textnormal{(iii)}] \textbf{Equal spacing\footnote{Under uniform constraint multiplicity.}:}
        Every constraint is active, $s_{i+1}-s_i=\gamma$, giving
        $s_i=\gamma\,(i-\tfrac{N-1}{2})$.
\end{enumerate}
\end{proposition}

\begin{proof}
\textbf{Step 1 (collinearity):}
Split each embedding into its component along $w$ and orthogonal to it,
\[
  h_{e_i} \;=\; s_i\,w + h_{e_i}^{\perp},
  \qquad \langle w,\,h_{e_i}^{\perp}\rangle = 0,
\]
so that $\lVert h_{e_i}\rVert^2 = s_i^2 + \lVert h_{e_i}^{\perp}\rVert^2$.
The constraints \eqref{eq:margin} depend only on the scores $s_i$, never on the
orthogonal parts. Setting every $h_{e_i}^{\perp}=0$ therefore preserves
feasibility while strictly lowering the objective unless it is already zero.
Hence $h_{e_i}=s_i w$ at any optimum.

The problem collapses to a convex quadratic program in the scalars alone,
\[
  \min_{\{s_i\}}\ \sum_i s_i^2
  \qquad\text{s.t.}\qquad
  s_{i+1}-s_i\ge\gamma,
\]
which has a unique minimizer.

\medskip
\textbf{Step 2 (equal spacing):}
Suppose for contradiction that some gap is slack: 
$s_{k+1} - s_k > \gamma$ for some $0 \leq k < N-1$. 
Consider shifting the upper block $\{s_{k+1}, \ldots, s_{N-1}\}$ down by $\epsilon > 0$ 
and the lower block $\{s_0, \ldots, s_k\}$ up by $\delta > 0$, where
\begin{equation*}
    (N - 1 - k)\,\epsilon = (k + 1)\,\delta,
\end{equation*}
so the total shift is zero. Every gap within each block is unchanged, the slack gap 
at $k$ widens, and all other tight gaps remain at $\gamma$, so the move is feasible. 
The change in the objective is
\begin{equation*}
    \Delta = -2\epsilon \sum_{m > k} s_m + 2\delta \sum_{m \leq k} s_m + O(\epsilon^2).
\end{equation*}
Substituting $\delta = \frac{(N-1-k)}{(k+1)}\epsilon$ and taking $\epsilon \to 0^+$, 
the sign of $\Delta$ is determined by
\begin{equation*}
    (k+1)\sum_{m > k} s_m - (N-1-k)\sum_{m \leq k} s_m.
\end{equation*}
If this quantity is positive, the original shift decreases the objective; 
if negative, reversing the shift (moving the upper block up and lower block down) 
decreases it. It is zero only when
\begin{equation*}
    \frac{1}{N-1-k}\sum_{m > k} s_m = \frac{1}{k+1}\sum_{m \leq k} s_m,
\end{equation*}
i.e.\ the two block means are equal, which combined with $s_{k+1} > s_k + \gamma > s_k$ 
is impossible for $\gamma > 0$. In either case a feasible descent direction exists, 
contradicting optimality. Every gap is therefore tight: $s_{i+1} - s_i = \gamma$ 
for all $i$.

\medskip
\textbf{Step 3 (monotonicity and centering):}
Equal gaps give $s_i = s_0 + \gamma i$, which is increasing in $i$. Minimizing
$\sum_i (s_0+\gamma i)^2$ over the offset $s_0$ yields
$s_i=\gamma\,(i-\tfrac{N-1}{2})$.
\end{proof}

\subsection*{What the proposition explains}

Proposition~\ref{prop:line} recovers our central findings from first principles.

Off-axis structure costs norm without affecting the fit, so regularization
collapses the embeddings onto a single axis (claim (i)) and that axis is
the rank axis (claim (ii)).

This is also \emph{why weight decay is necessary}: with $\lambda=0$ there is no
pressure toward the line, matching our observation that unregularized runs never
organize. Claims (i) and (ii) appear directly in Figure~\ref{fig:1}, where the
peak-checkpoint embeddings form a monotone curve with $r_{\mathrm{rank}}=0.96$.

\subsection*{End-anchoring: a predicted break from uniform spacing}

The equal-spacing claim (iii) assumes every entity faces the same number of
constraints. It does not.

Interior entities $e_i$ (with $0<i<N-1$) appear in \emph{two} adjacent
constraints, with $e_{i-1}$ and $e_{i+1}$, and are pulled toward the centre from
both sides. The endpoints $e_0$ and $e_{N-1}$ appear in only \emph{one}, and can
drift outward at lower marginal norm cost.

The minimum-norm solution therefore \emph{stretches at the ends and compresses
in the middle}. This is exactly what Figure~\ref{fig:1} shows: interior
ranks bunch together while $e_0$ and $e_{N-1}$ extend past the linear trend, and
the geometric separation $|\pi_i-\pi_j|$ accelerates at the largest distances
(Figure~\ref{fig:2}, (b)).

This same asymmetry is the geometric form of the \emph{end-anchor effect} (\cite{Holyoak1981-kx}), a
second canonical signature of transitive inference in humans and animals, in
which judgements involving the terminal items are made faster and more reliably
than those confined to the interior.

\subsection*{Scope}

The proposition characterizes the regularized \emph{optimum} under a linear
readout. Two caveats follow.

First, it does not describe the optimization \emph{trajectory}: the transient,
grokking-like dynamics of Figure~\ref{fig:3} are a property of that
trajectory and lie outside this account.

Second, the attention and MLP sub-layers of the trained model are abstracted into
the single fixed readout $w$. We therefore regard the argument as an explanation
of \emph{which} solution is preferred, empirically corroborated by the recovered
geometry, and leave a trajectory-level theory, and its possible link to
description-length accounts of weight decay to future work.
\section{Task and Methods}
\label{sec:methods}

\paragraph{Setup:}
Let $E=\{e_0,\dots,e_{N-1}\}$ carry a hidden total order
$e_0 \prec e_1 \prec \cdots \prec e_{N-1}$. Each entity is an atomic token and
the order is never given explicitly. A comparison is a triple $(e_i,\,?,\,e_j)$
whose target is the relation
\begin{equation}
r(e_i,e_j)=
\begin{cases}
\;\prec ;& i<j,\\[2pt]
\;\succ ;& i>j.
\end{cases}
\end{equation}

\paragraph{Train/OOD split:}
The model is trained \emph{only} on adjacent comparisons, in both directions,
\begin{equation}
\mathcal{D}_{\mathrm{train}}=
\bigl\{(e_i,e_{i+1}),\,(e_{i+1},e_i)\ :\ 0\le i<N-1\bigr\},
\end{equation}
and evaluated on the disjoint held-out set of all distant pairs,
\begin{equation}
\mathcal{D}_{\mathrm{OOD}}=
\bigl\{(e_i,e_j)\ :\ |i-j|>1\bigr\},
\qquad
\mathcal{D}_{\mathrm{train}}\cap\mathcal{D}_{\mathrm{OOD}}= \varnothing 
\end{equation}
Since no distant pair is ever observed, success on $\mathcal{D}_{\mathrm{OOD}}$
requires reconstructing the global order $\prec$ from local comparisons alone, the
computational form of the Bryant-Trabasso paradigm.

\paragraph{Model:}
We use a single-layer Transformer $f_\theta$ with entity embeddings
$\{h_e\in\mathbb{R}^{d}\}$, two attention heads, and unembedding head $W_U$.
The relation is read from the final-position residual; we train with
cross-entropy on the relation token using AdamW (\cite{loshchilov2019decoupledweightdecayregularization}) with weight decay $\lambda$.

\paragraph{Measuring the number line:}
Let $H=[\,h_{e_0};\dots;h_{e_{N-1}}\,]\in\mathbb{R}^{N\times d}$ be the entity
embeddings and $u_1$ their first principal component, with eigenvalues
$\lambda_1\ge\dots\ge\lambda_d$. We quantify ordinal structure three ways:
\begin{enumerate}\itemsep2pt
  \item \textbf{Linearity:} the variance fraction on $u_1$,
        $\rho_1=\lambda_1/\sum_k \lambda_k$, large when embeddings lie near a line.
  \item \textbf{Rank alignment:} the correlation between projected position
        $\pi_i=\langle h_{e_i},u_1\rangle$ and true rank $i$,
        \begin{equation}
            r_{\mathrm{rank}}=\operatorname{corr}\bigl(\{\pi_i\}_{i},\{i\}_{i}\bigr),
        \end{equation}
        near $\pm1$ when the line recovers the hidden order.
  \item \textbf{Geometric separation:} $g_{ij}=|\pi_i-\pi_j|$, the distance
        between two items on the recovered line.
\end{enumerate}

\paragraph{Behavioral signatures:}
Because accuracy saturates once the order is learned, we measure graded effects
that persist at ceiling. For a query $(e_i,e_j)$ with correct and incorrect
relation tokens $c,\bar{c}$, the \emph{decision confidence} is the logit margin
\begin{equation}
    m_{ij}=\bigl(W_U z_{ij}\bigr)_c-\bigl(W_U z_{ij}\bigr)_{\bar{c}},
\end{equation}
where $z_{ij}$ is the final-position residual. The symbolic distance effect
predicts that both $m_{ij}$ and $g_{ij}$ increase with rank distance $|i-j|$;
we report their Pearson correlation with $|i-j|$ over $\mathcal{D}_{\mathrm{OOD}}$.

\section{Synthetic Experiment on Small Transformers}

The synthetic model is a single-layer, single-block Transformer with hidden
dimension\footnote{We use $d = 64$ for reliable grokking when training with more entities (N > 15).} $d=32$, $2$ attention heads, a $4d$-wide GELU MLP, learned positional
embeddings, and pre-LayerNorm; the relation is decoded from the final-position
residual by a linear unembedding head ($\approx 5\times10^{4}$ parameters).
We optimize the relation-token cross-entropy with AdamW (learning rate
$10^{-3}$, batch size $64$) and weight decay $\lambda$ swept per chain length
$N$, training for up to $5\times10^{4}$ steps and reporting all quantities at
the peak-OOD accuracy checkpoint. The experiment is run on a 16GB M2 Macbook Air.

We run the experiment for N = 10, 15, and 20 entities. The figures for N = 10, 20 can be found in the appendix. The following discussion is a qualitative breakdown of the trends seen in all three results. 

\begin{figure}[htbp]
    \centering
    \includegraphics[width=\textwidth]{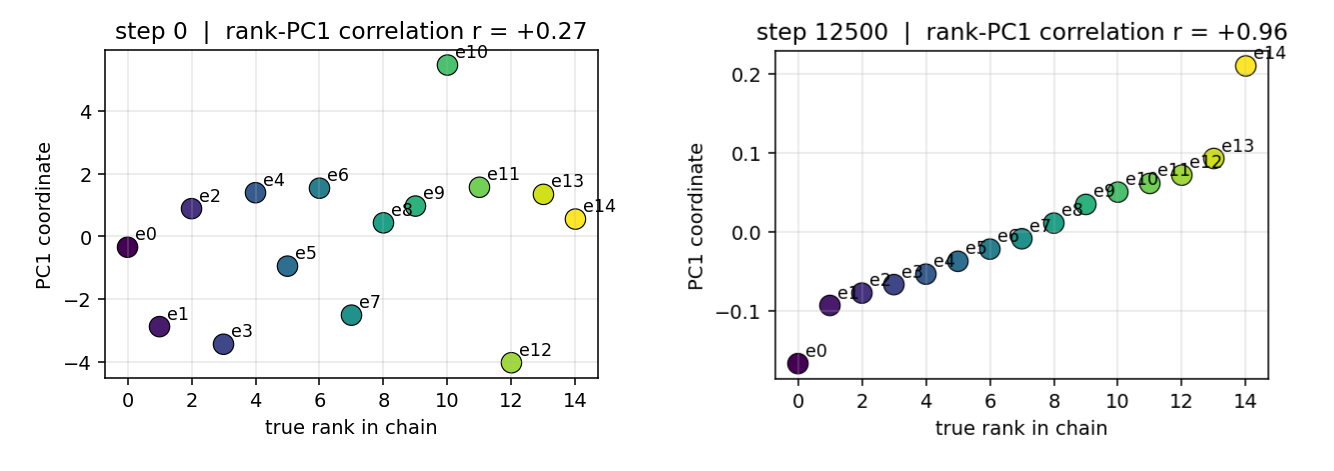} 
    \caption{First principal projection vs. true rank in chain for N = 15 entities.}
    \label{fig:1}
\end{figure}

\begin{figure}[htbp]
    \centering
    \includegraphics[width=\textwidth]{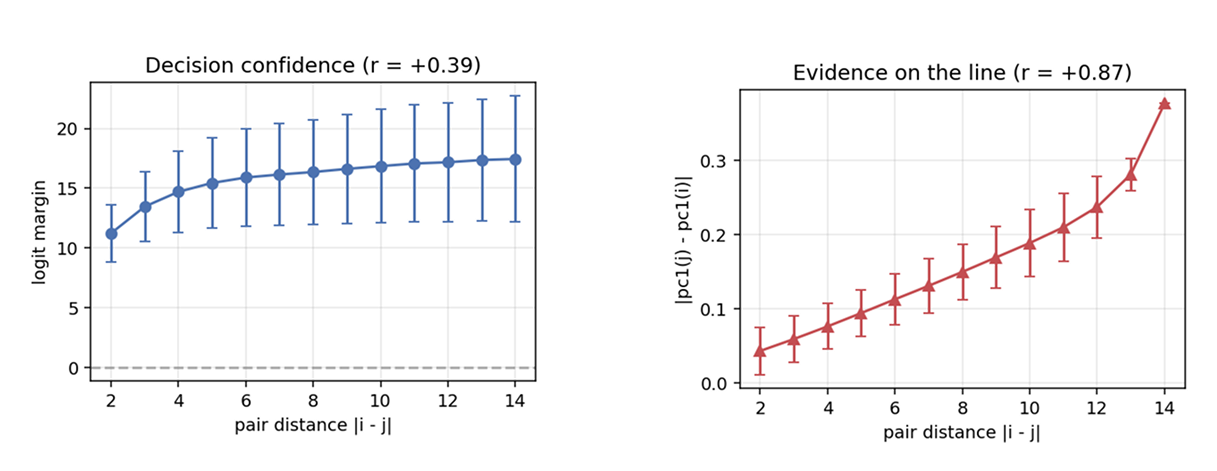} 
    \caption{(a) Decision confidence as a function of distance between compared entities. (b) PC1 distance as a function of pair distance.}
    \label{fig:2}
\end{figure}

\begin{figure}[htbp]
    \centering
    \includegraphics[width=\textwidth]{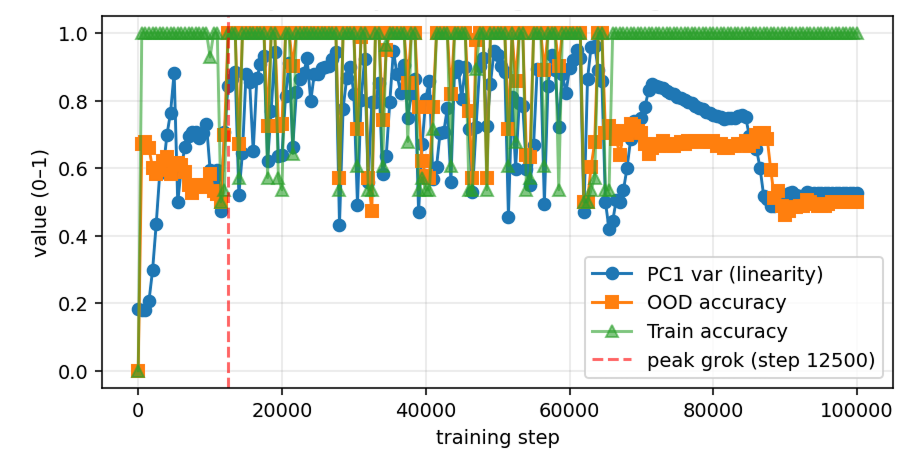} 
    \caption{Training dynamics for N = 15 entities.}
    \label{fig:3}
\end{figure}

Figure \ref{fig:1} (rank vs.\ PC1) makes the linearity argument explicit by plotting each entity's projection onto the first principal component against its true rank: at initialization the relationship is essentially random, whereas at peak the points form a clean monotonic staircase, confirming that the principal axis recovers the hidden total order almost exactly. Figure \ref{fig:2} tests for the symbolic distance effect using measures that remain informative once accuracy saturates: decision confidence and geometric separation on the recovered line, both increase monotonically with rank distance, exactly as predicted if comparisons are read off a mental number line.

Figure \ref{fig:3} overlays PC1 variance, OOD accuracy, and train accuracy across training; train accuracy saturates almost immediately while OOD generalization and embedding linearity rise together and later, and crucially both degrade after the peak (dashed line), illustrating that the structured solution is transient under fixed weight decay — which is why all analyses are reported at the peak-OOD checkpoint.

\section{Scaling to Large Language Models}

The synthetic results raise the obvious question: is the rank line an artifact
of a tiny model trained on a single hand-built order, or does the same geometry
survive in large models trained on natural language? A pretrained LLM is never
given an explicit comparison task, yet its training corpus is saturated with
ordinal structure (number sequences, calendars, magnitude adjectives). If the
rank-line account captures something general, fragments of it should be legible
in the \emph{frozen} representations of an off-the-shelf model (\cite{gurnee2024languagemodelsrepresentspace}).

We probe \texttt{Qwen2.5-1.5B} on three ordinal domains chosen to vary in how
cleanly they form a total order: the digits $0$-$9$, the twelve months of the
year, and a set of size adjectives ordered by magnitude
(\emph{tiny}, \emph{little}, \emph{small}, \dots, \emph{enormous}, \emph{gigantic}).
Each item is tokenized in isolation with special tokens (BOS, etc.) dropped, and
its layer-$\ell$ representation is the mean of the residual-stream activations over
the item's sub-token positions at that layer: a single-token item contributes its
one vector, while a word that splits (e.g.\ \texttt{Sept}+\texttt{ember}) contributes
the average of its pieces. We mean-pool rather than take the final sub-token, the
standard alternative, which leans on the final position having attended over the
whole word under causal masking because averaging privileges no single piece;
neither choice is wrong, and we report the mean-pooled result. Stacking these into
$H \in \mathbb{R}^{N \times d}$, we apply the diagnostics of Section~4: PCA for
linearity, the rank-PC1 correlation $r_{\text{rank}}$ (Figure~\ref{fig:llm_pca}),
and because a model carrying many other features need not place the ordinal axis
on its leading component, the best $|r_{\text{rank}}|$ over the top three PCs as a
function of depth (Figure~\ref{fig:llm_linearity}), together with a leave-one-out
(LOO) (\cite{belinkov2021probingclassifierspromisesshortcomings}) linear probe regressing the hidden state onto rank, reported as $R^2$. To ask
whether distinct domains share an encoding direction, we also compute the cosine
similarity between the probe directions fit on each pair of domains, layer by layer.

\begin{figure}[htbp]
    \centering
    \includegraphics[width=0.7\textwidth, trim={0 0 0 0.62cm}, clip]{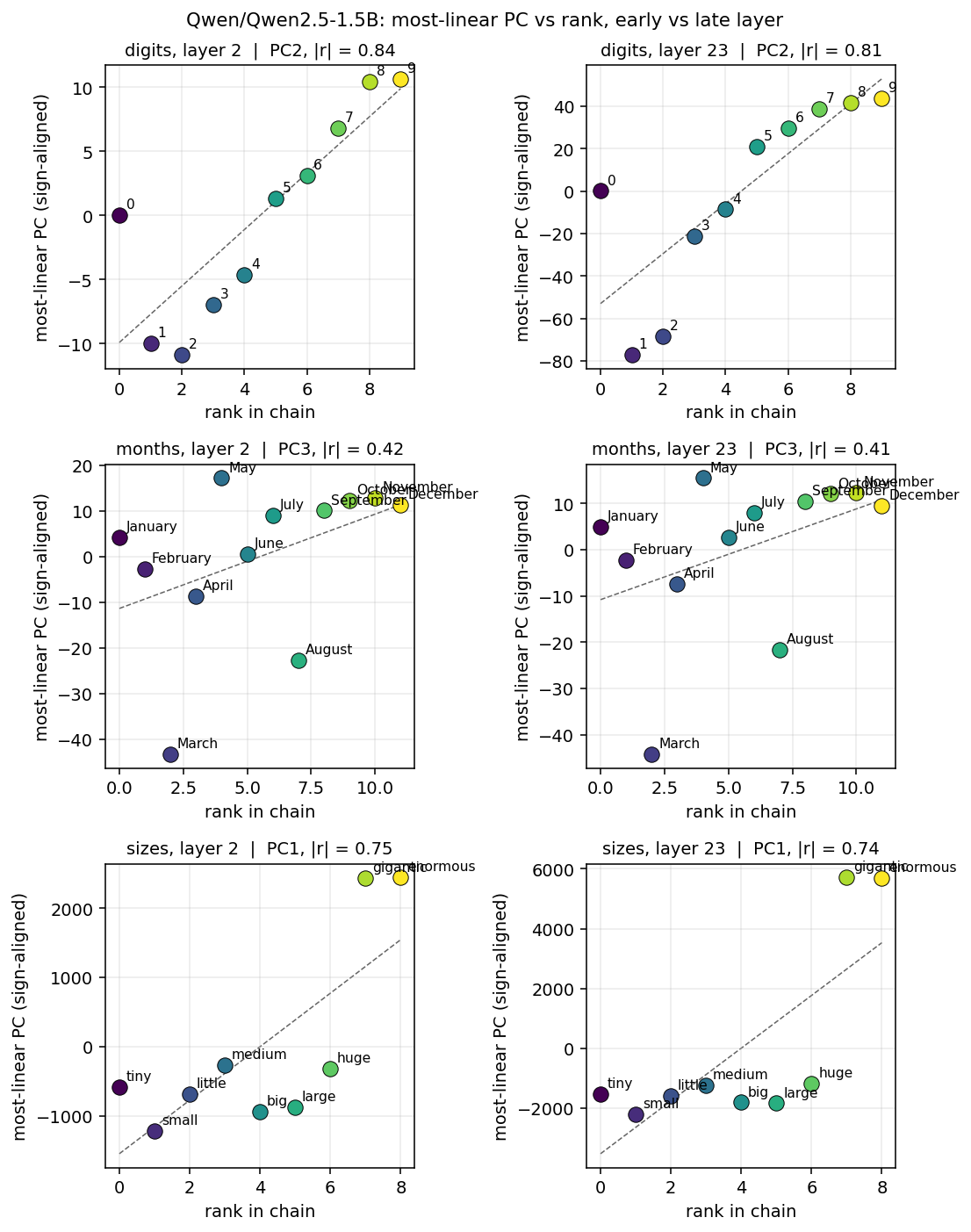} 
    \caption{PCA plots for the three ordinal domains: Digits (0-9), months of the year, and size adjectives, in early and late layers.}
    \label{fig:llm_pca}
\end{figure}

\paragraph{Ordinal structure is present but domain-dependent.}
All three domains exhibit a positive rank-PC correlation that is stable between
early and late layers (Figure~\ref{fig:llm_pca}): digits are recovered most cleanly
($r \approx 0.81$ on PC2), sizes moderately, and months most weakly
($r = 0.41$). The size adjectives form the tightest ordering, with the
diminutives at one extreme and \emph{enormous}/\emph{gigantic} at the other,
consistent with magnitude being an unusually direct, frequently verbalised ordinal
relation. Digits show a characteristic distortion: the token ``0'' sits far from
$1$-$9$ along PC1, so the ordinal structure of the remaining digits is carried by a
lower component, perhaps due to the vastly different use-cases of zero as a digit relative to the other digits.

\paragraph{The cyclic domain resists the line.}
Months are the weakest of the three, which is exactly what the rank-line picture
predicts. A calendar is cyclic; December is adjacent to January. A cycle
cannot be embedded isometrically on a single linear axis. The weak,
layer-stable months correlation is therefore not a failure of the probe but a
signature that the underlying relation is not a total order, directly addressing
the cyclic-relation question raised in Section~7.

\paragraph{Structure is layer-stable, not late-emerging.}
Across all three domains the rank-PC correlation and probe $R^2$ are roughly flat
over network depth (Figure~\ref{fig:llm_linearity}), in contrast to the late,
transient emergence seen in the synthetic model (Figure~3). This is expected: the
LLM is not \emph{learning} the order during measurement, the ordinal geometry is
already fixed in its pretrained weights, so we observe its steady state rather than
its formation. The digit probe $R^2$ is the noisiest, dipping sharply (even below
zero under LOO) in the middle layers, where the ``0'' outlier most distorts the
fit.

\begin{figure}[htbp]
    \centering
    \includegraphics[width=\textwidth, trim={0 0 0 0.62cm}, clip]{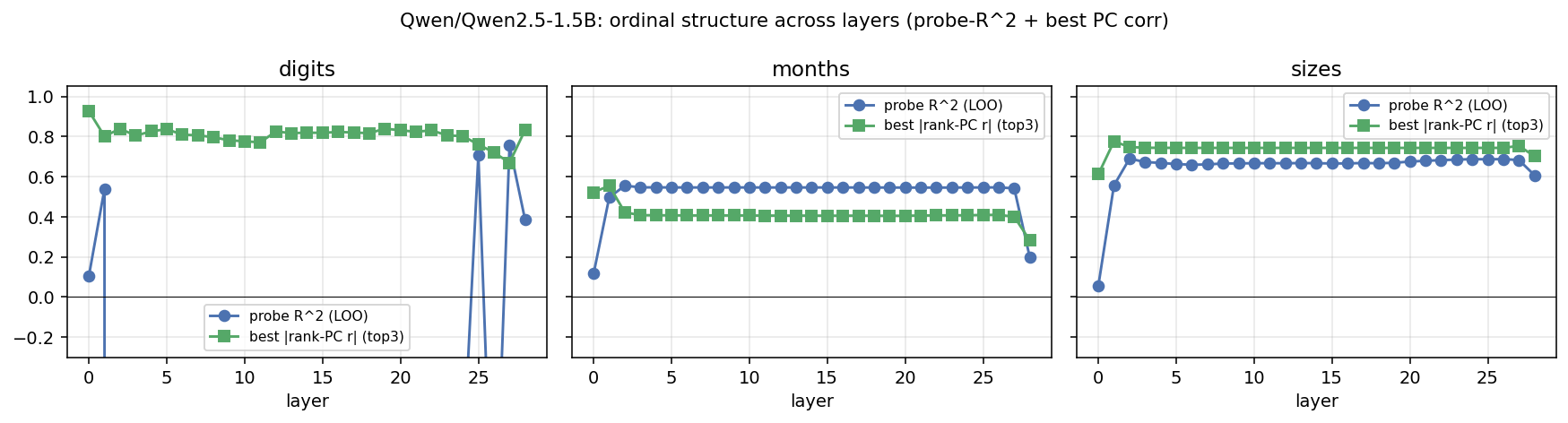} 
    \caption{Probe R² (LOO) and best rank–PC correlation as a function of layer depth, for each ordinal domain.}
    \label{fig:llm_linearity}
\end{figure}
\paragraph{Each domain occupies its own axis.}
Figure~\ref{fig:llm_shared} shows that the probe directions for different domains
are nearly orthogonal at every layer ($|\cos|<0.2$ throughout, with only a mild
rise in the final layers). The model does not appear to encode a single,
domain-general ``ordinality'' direction; instead each ordinal concept occupies its
own subspace, even while all three individually obey the same low-dimensional,
rank-aligned geometry. This is consistent with the synthetic theory, which fixes a
separate readout $w$ per task and makes no claim that independent orders share an
axis.

\begin{figure}[htbp]
    \centering
    \includegraphics[width=0.7\textwidth, trim={0 0 0 0.62cm}, clip]{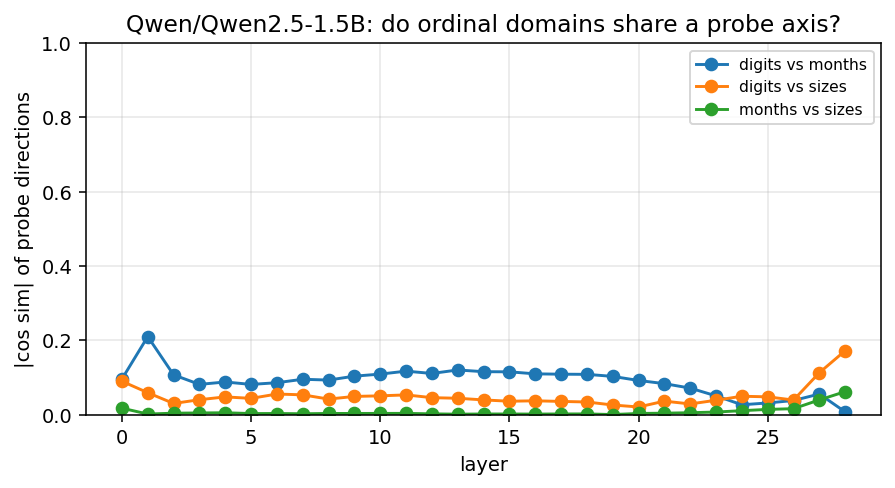} 
    \caption{Cosine similarity between per-domain probe directions, by layer. }
    \label{fig:llm_shared}
\end{figure}

\paragraph{Takeaway.}
The qualitative prediction of the rank-line account survives the jump from a
$\sim$$5\times10^4$-parameter toy model to a $1.5$B-parameter pretrained LLM:
ordinal concepts are stored along approximately linear, rank-aligned manifolds
whose cleanliness tracks how genuinely total the underlying order is
(sizes $>$ digits $>$ the cyclic months). What does \emph{not} transfer is
universality: the model allocates a distinct axis per domain rather than reusing
one. As in the synthetic study these observations are correlational; we leave
causal interventions and a broader model and domain sweep to future work.

\section{Future Work}
The results in this paper point to several natural extensions. The observations are correlational, we make no claims on the causality of these observations. The effect of different depths and architectures on these results remain an interesting avenue, which also motivate theoretical work in trajectory level transient grokking dynamics. Finally, the LLM study is observational and limited to one 1.5B model and three hand-chosen domains; scaling across model sizes, probing cyclic and partial orders directly (the months result suggests a loop, not a line), and testing whether interventions on the recovered axis change model behaviour would together turn the present geometric description into a mechanistic and causal account.
\newpage
\bibliographystyle{plainnat} 
\bibliography{references}    

@ARTICLE{Bryant1971-tp,
  title    = "Transitive inferences and memory in young children",
  author   = "Bryant, P E and Trabasso, T",
  journal  = "Nature",
  volume   =  232,
  number   =  5311,
  pages    = "456--458",
  month    =  aug,
  year     =  1971,
  address  = "England",
  language = "en"
}

@Article{Steirn1995,
author={Steirn, Janice N.
and Weaver, Janice E.
and Zentall, Thomas R.},
title={Transitive inference in pigeons: Simplified procedures and a test of value transfer theory},
journal={Animal Learning {\&} Behavior},
year={1995},
month={Mar},
day={01},
volume={23},
number={1},
pages={76-82},
issn={1532-5830},
doi={10.3758/BF03198018},
url={https://doi.org/10.3758/BF03198018}
}

@misc{minegishi2026emergentanalogicalreasoningtransformers,
      title={Emergent Analogical Reasoning in Transformers}, 
      author={Gouki Minegishi and Jingyuan Feng and Hiroki Furuta and Takeshi Kojima and Yusuke Iwasawa and Yutaka Matsuo},
      year={2026},
      eprint={2602.01992},
      archivePrefix={arXiv},
      primaryClass={cs.AI},
      url={https://arxiv.org/abs/2602.01992}, 
}

@article{power,
  author       = {Alethea Power and
                  Yuri Burda and
                  Harri Edwards and
                  Igor Babuschkin and
                  Vedant Misra},
  title        = {Grokking: Generalization Beyond Overfitting on Small Algorithmic Datasets},
  journal      = {CoRR},
  volume       = {abs/2201.02177},
  year         = {2022},
  url          = {https://arxiv.org/abs/2201.02177},
  eprinttype   = {arXiv},
  eprint       = {2201.02177},
  bibsource    = {dblp computer science bibliography, https://dblp.org}
}

@article{RILEY19761,
title = {The representation of comparative relations and the transitive inference task},
journal = {Journal of Experimental Child Psychology},
volume = {22},
number = {1},
pages = {1-22},
year = {1976},
issn = {0022-0965},
doi = {https://doi.org/10.1016/0022-0965(76)90085-0},
url = {https://www.sciencedirect.com/science/article/pii/0022096576900850},
author = {Christine A Riley}
}

@misc{nanda2023,
      title={Progress measures for grokking via mechanistic interpretability}, 
      author={Neel Nanda and Lawrence Chan and Tom Lieberum and Jess Smith and Jacob Steinhardt},
      year={2023},
      eprint={2301.05217},
      archivePrefix={arXiv},
      primaryClass={cs.LG},
      url={https://arxiv.org/abs/2301.05217}, 
}

@misc{park2025iclrincontextlearningrepresentations,
      title={ICLR: In-Context Learning of Representations}, 
      author={Core Francisco Park and Andrew Lee and Ekdeep Singh Lubana and Yongyi Yang and Maya Okawa and Kento Nishi and Martin Wattenberg and Hidenori Tanaka},
      year={2025},
      eprint={2501.00070},
      archivePrefix={arXiv},
      primaryClass={cs.CL},
      url={https://arxiv.org/abs/2501.00070}, 
}

@ARTICLE{Holyoak1981-kx,
  title    = "A positional discriminability model of linear-order judgments",
  author   = "Holyoak, K J and Patterson, K K",
  journal  = "J Exp Psychol Hum Percept Perform",
  volume   =  7,
  number   =  6,
  pages    = "1283--1302",
  month    =  dec,
  year     =  1981,
  address  = "United States",
  language = "en"
}

@misc{elhage2022toymodelssuperposition,
      title={Toy Models of Superposition}, 
      author={Nelson Elhage and Tristan Hume and Catherine Olsson and Nicholas Schiefer and Tom Henighan and Shauna Kravec and Zac Hatfield-Dodds and Robert Lasenby and Dawn Drain and Carol Chen and Roger Grosse and Sam McCandlish and Jared Kaplan and Dario Amodei and Martin Wattenberg and Christopher Olah},
      year={2022},
      eprint={2209.10652},
      archivePrefix={arXiv},
      primaryClass={cs.LG},
      url={https://arxiv.org/abs/2209.10652}, 
}

@misc{loshchilov2019decoupledweightdecayregularization,
      title={Decoupled Weight Decay Regularization}, 
      author={Ilya Loshchilov and Frank Hutter},
      year={2019},
      eprint={1711.05101},
      archivePrefix={arXiv},
      primaryClass={cs.LG},
      url={https://arxiv.org/abs/1711.05101}, 
}

@ARTICLE{Mcgonigle1977-qz,
  title    = "Are monkeys logical?",
  author   = "Mcgonigle, Brendan O and Chalmers, Margaret",
  journal  = "Nature",
  volume   =  267,
  number   =  5613,
  pages    = "694--696",
  month    =  jun,
  year     =  1977
}

@ARTICLE{Davis1992-ui,
  title    = "Transitive inference in rats (Rattus norvegicus)",
  author   = "Davis, H",
  journal  = "J Comp Psychol",
  volume   =  106,
  number   =  4,
  pages    = "342--349",
  month    =  dec,
  year     =  1992,
  address  = "United States",
  language = "en"
}

@inproceedings{mikolov-etal-2013-linguistic,
    title = "Linguistic Regularities in Continuous Space Word Representations",
    author = "Mikolov, Tomas  and
      Yih, Wen-tau  and
      Zweig, Geoffrey",
    editor = "Vanderwende, Lucy  and
      Daum{\'e} III, Hal  and
      Kirchhoff, Katrin",
    booktitle = "Proceedings of the 2013 Conference of the North {A}merican Chapter of the Association for Computational Linguistics: Human Language Technologies",
    month = jun,
    year = "2013",
    address = "Atlanta, Georgia",
    publisher = "Association for Computational Linguistics",
    url = "https://aclanthology.org/N13-1090/",
    pages = "746--751"
}

@misc{neyshabur2015searchrealinductivebias,
      title={In Search of the Real Inductive Bias: On the Role of Implicit Regularization in Deep Learning}, 
      author={Behnam Neyshabur and Ryota Tomioka and Nathan Srebro},
      year={2015},
      eprint={1412.6614},
      archivePrefix={arXiv},
      primaryClass={cs.LG},
      url={https://arxiv.org/abs/1412.6614}, 
}

@misc{belinkov2021probingclassifierspromisesshortcomings,
      title={Probing Classifiers: Promises, Shortcomings, and Advances}, 
      author={Yonatan Belinkov},
      year={2021},
      eprint={2102.12452},
      archivePrefix={arXiv},
      primaryClass={cs.CL},
      url={https://arxiv.org/abs/2102.12452}, 
}

@misc{gurnee2024languagemodelsrepresentspace,
      title={Language Models Represent Space and Time}, 
      author={Wes Gurnee and Max Tegmark},
      year={2024},
      eprint={2310.02207},
      archivePrefix={arXiv},
      primaryClass={cs.LG},
      url={https://arxiv.org/abs/2310.02207}, 
}
\newpage

\appendix
\section{Technical Appendices and Supplementary Material}
\subsection{Data for N = 10 entities} 
\begin{figure}[htbp]
    \centering
    \includegraphics[width=\textwidth]{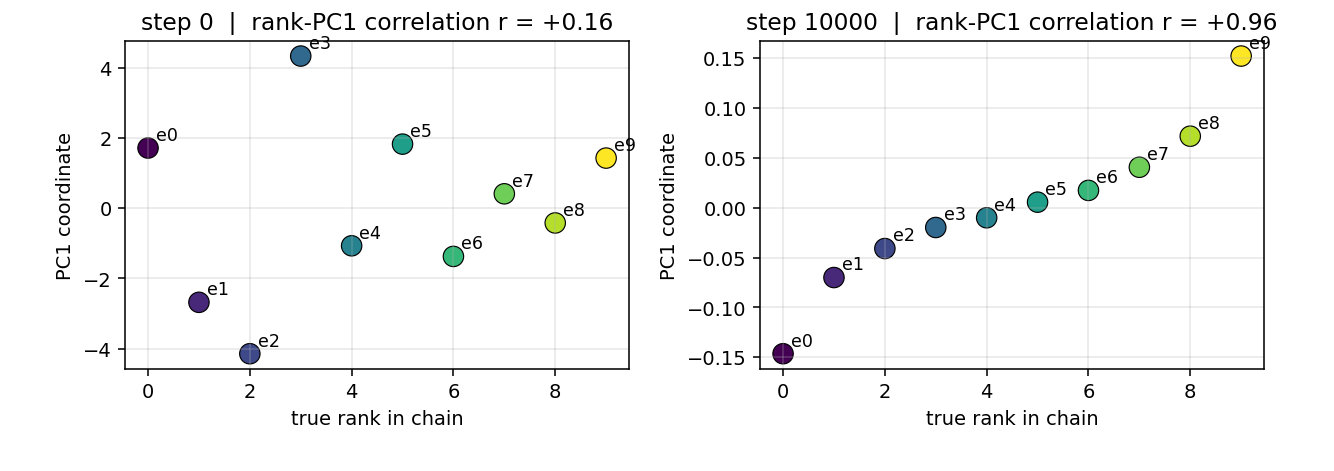} 
    \caption{First principal projection vs. true rank in chain for N = 10 entities.}
    \label{fig:4}
\end{figure}

\begin{figure}[htbp]
    \centering
    \includegraphics[width=\textwidth]{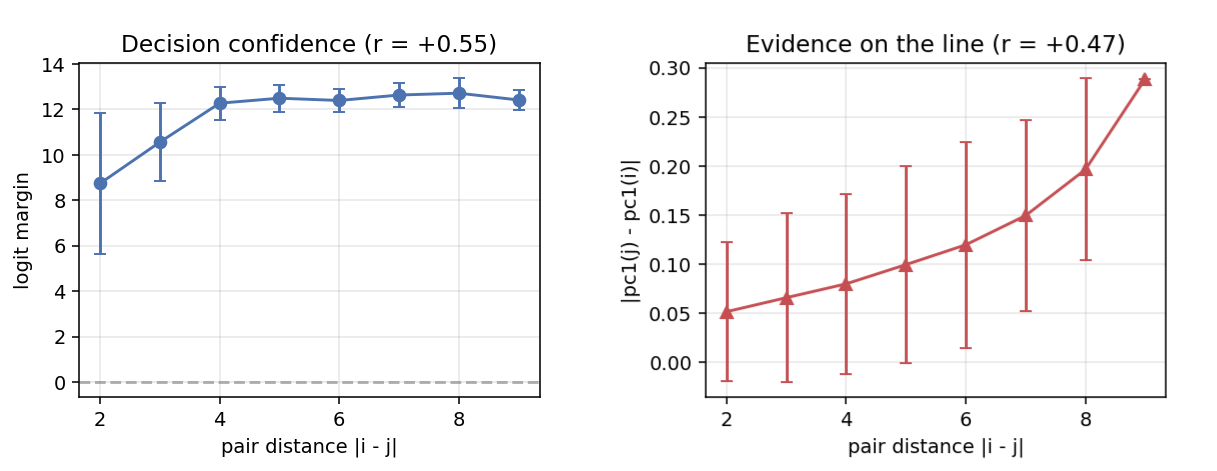} 
    \caption{(a) Decision confidence as a function of distance between compared entities. (b) PC1 distance as a function of pair distance.}
    \label{fig:5}
\end{figure}

\begin{figure}[htbp]
    \centering
    \includegraphics[height=0.22\textheight, keepaspectratio]{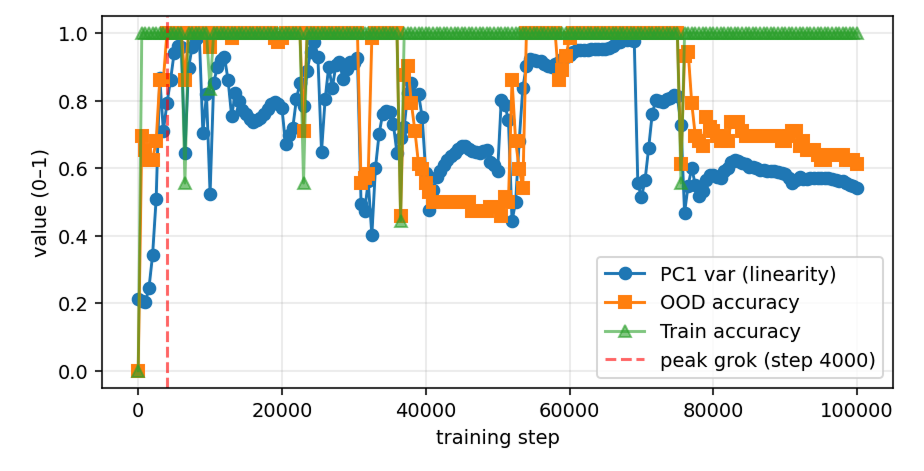}
    \caption{Training dynamics for N = 10 entities.}
    \label{fig:6}
\end{figure}
\newpage

\subsection{Data for N = 20 entities} 
\begin{figure}[htbp]
    \centering
    \includegraphics[width=\textwidth]{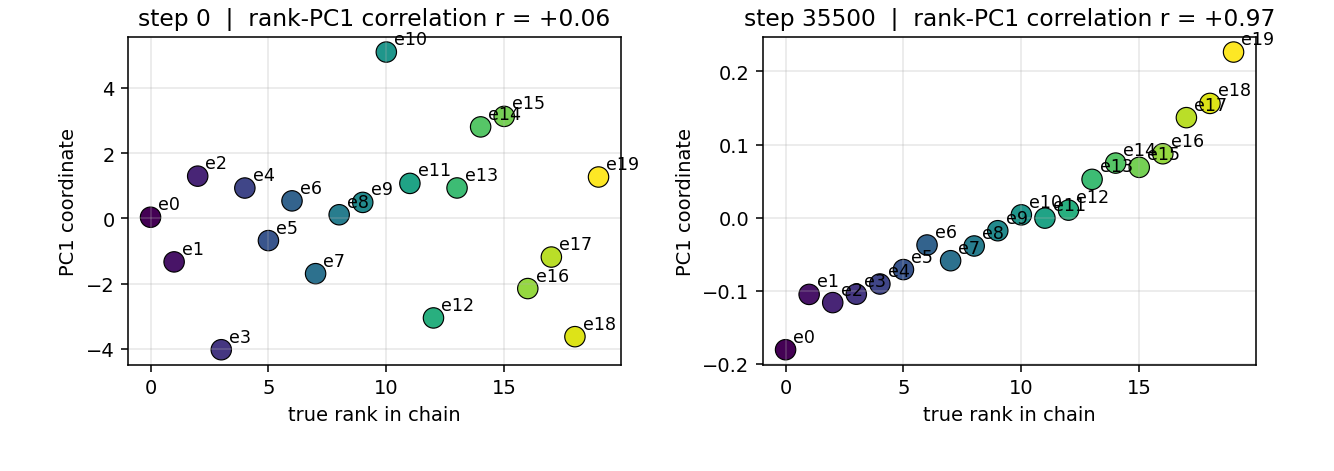} 
    \caption{First principal projection vs. true rank in chain for N = 20 entities.}
    \label{fig:7}
\end{figure}

\begin{figure}[htbp]
    \centering
    \includegraphics[width=\textwidth]{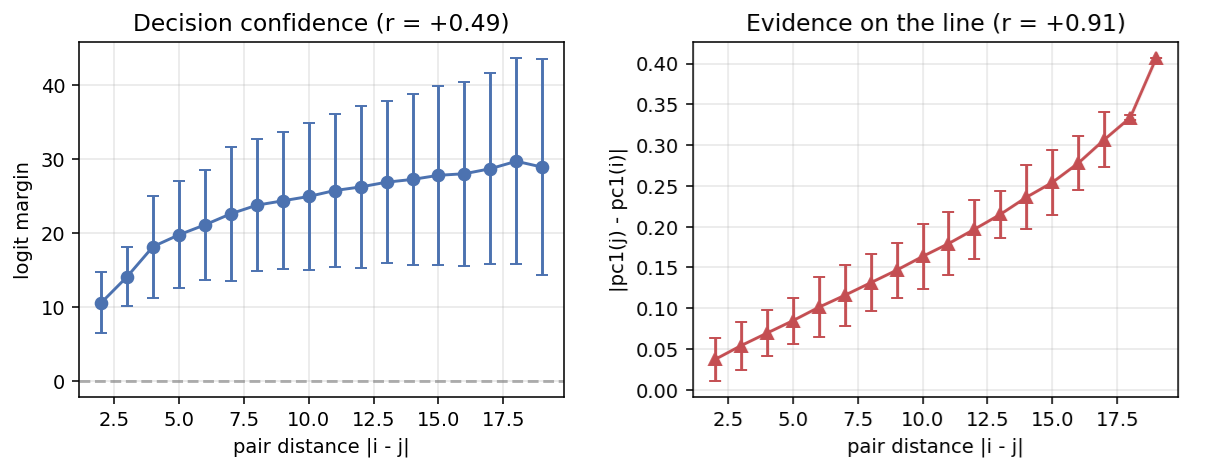} 
    \caption{(a) Decision confidence as a function of distance between compared entities. (b) PC1 distance as a function of pair distance.}
    \label{fig:8}
\end{figure}

\begin{figure}[htbp]
    \centering
    \includegraphics[height=0.22\textheight, keepaspectratio]{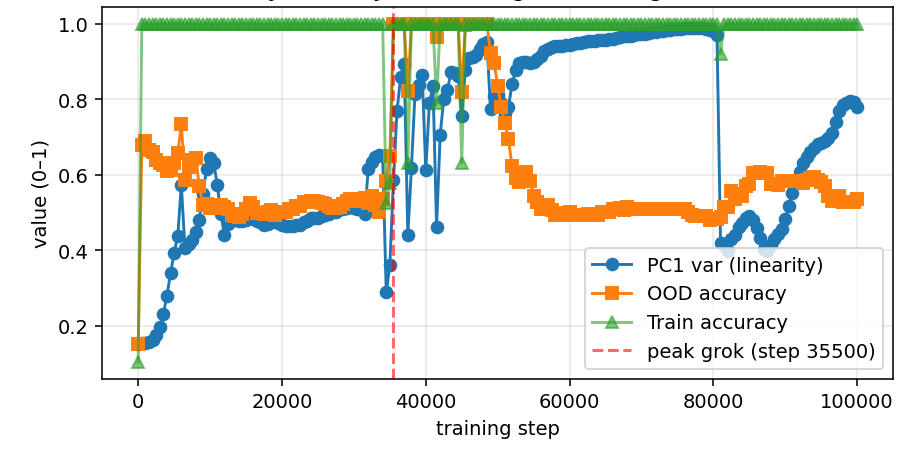} 
    \caption{Training dynamics for N = 20 entities.}
    \label{fig:9}
\end{figure}

\newpage

\end{document}